\pdfoutput=1

\documentclass{article}
\usepackage{ijcai19}

\def\suppmat{1}
\newcommand{\extlink}{https://gofile.io/?c=NsXVpU}
\usepackage{times}
\usepackage{nimbusmononarrow}
\usepackage{soul}
\usepackage{url}
\usepackage[hidelinks]{hyperref}
\usepackage[utf8]{inputenc}
\usepackage[small]{caption}
\usepackage{xcolor}
\usepackage{graphicx}
\usepackage{amsmath}
\usepackage{amsthm}
\usepackage{amssymb}
\usepackage{booktabs} 

\definecolor{dkblue}{cmyk}{1,.54,.04,.19}
\hypersetup{
      bookmarks=true,         
      unicode=false,          
      pdftoolbar=true,        
      pdfmenubar=true,        
      pdffitwindow=false,     
      pdfstartview={FitH},    
      pdfcreator={pdflatex},   
      pdfproducer={Producer}, 
      pdfnewwindow=true,      
      colorlinks=true,        
      linkcolor=dkblue,       
      citecolor=dkblue,       
      filecolor=dkblue,       
      urlcolor=dkblue,        
}

\usepackage{listings} 
\usepackage{mathtools}
\usepackage[capitalise]{cleveref}
\crefname{listing}{Algorithm}{Algorithms}
\Crefname{listing}{Algorithm}{Algorithms}

\usepackage{wrapfig}

\theoremstyle{definition}
\newtheorem{theorem}{Theorem}
\newtheorem{lemma}[theorem]{Lemma}

\newtheorem{remark}[theorem]{Remark}

\definecolor{darkred}{rgb}{.7,0,0}
\definecolor{darkgreen}{rgb}{0,.5,0}
\definecolor{darkblue}{rgb}{0,0,.8}
\definecolor{darkcyan}{rgb}{0,0.6,.6}
\definecolor{darkorange}{rgb}{.8,.4,0}
\definecolor{gray}{rgb}{.4,.4,.4}

\newcommand{\Nopt}{N_*}

\newcommand{\comment}[1]{\textcolor{darkblue}{(\emph{#1})}}

\newcommand{\Naturals}{\mathbb{N}_1}
\newcommand{\Nonnegints}{\mathbb{N}_0}
\newcommand{\Reals}{\mathbb{R}}

\newcommand{\eg}{\emph{e.g.}, }

\newcommand{\eps}{\varepsilon}

\newcommand{\ceiling}[1]{\left\lceil#1\right\rceil}
\newcommand{\floor}[1]{\left\lfloor#1\right\rfloor}

\newcommand{\citet}[1]{\citeauthor{#1}~[\citeyear{#1}]}

\newcommand{\magic}{\text{A6519}}

\newcommand{\astar}{A$^*$}
\newcommand{\ida}{IDA$^*$}
\newcommand{\eda}{EDA$^*$}
\newcommand{\idacr}{\ida\_CR}
\newcommand{\zoomer}{Zoomer}
\newcommand{\zzzlong}{ZigZagZoomer}
\newcommand{\zzz}{Z3} 
\newcommand{\deepening}{\texttt{DFS}} 

\newcommand{\defeq}{\ensuremath{\coloneqq}}

\newcommand{\node}{n}
\newcommand{\children}{\mathcal{C}}

\newcommand{\nodeset}{\mathcal{N}}

\newcommand{\thresh}{\theta} 
\newcommand{\fcost}{\ensuremath{f}} 
\newcommand{\cost}{\fcost}
\newcommand{\gcost}{\ensuremath{g}} 
\newcommand{\hcost}{\ensuremath{h}} 

\newcommand{\branch}{b} 
\newcommand{\openset}{\mathcal{M}} 

\newcommand{\zoomfact}{\omega_1}
\newcommand{\zzzfact}{\omega_2}
\newcommand{\nodecount}{N}
\newcommand{\goalset}{\mathcal{G}}
\newcommand{\magicshort}{A}

\newcommand{\fgap}{\delta} 
\newcommand{\fgapopt}{\delta^*} 
\newcommand{\fgapmin}{\delta_\text{min}} 

\newcommand{\threshup}{\thresh^{+}} 
\newcommand{\threshdown}{\thresh^{-}} 

\newcommand{\outabudget}{\texttt{"budget exceeded"}}
\newcommand{\upper}{\texttt{upper}}
\newcommand{\llower}{\texttt{lower}}

\newcommand{\ceilingpos}[1]{\ceiling{#1}^{\scriptscriptstyle +}} 

\title{Zooming Cautiously: \\
Linear-Memory Heuristic Search With Node Expansion Guarantees
}

\author{
Laurent Orseau$^1$, Levi H. S. Lelis$^2$, Tor Lattimore$^1$
\affiliations
$^1$DeepMind, UK\\
$^2$Universidade Federal de Viçosa, Brazil \\
\emails
\{lorseau, lattimore\}@google.com,
levi.lelis@ufv.br
}

\begin{document}

\maketitle

\begin{abstract}
We introduce\footnote{This paper and another independent IJCAI 2019 submission have been merged into a single paper that subsumes both of them~\cite{helmert2019ibex}. This paper is placed here only for historical context. Please only cite the subsuming paper.}
and analyze two parameter-free linear-memory tree search algorithms. Under mild assumptions we prove our algorithms are guaranteed to perform only a logarithmic factor more node expansions than \astar{}
when the search space is a tree.  Previously, the best guarantee for a linear-memory algorithm under similar assumptions was achieved by \ida{}, which in the worst case expands quadratically more nodes than in its last iteration. Empirical results support the theory and demonstrate the practicality and robustness of our algorithms. Furthermore, they are fast and easy to implement.%
\ifdefined\suppmat{}
\else
\footnote{The 1-page supplementary material can be found at \url{\extlink}.}
\fi%
\end{abstract}

\section{Introduction}

The \astar{} algorithm \cite{hart1968astar} is optimal in the sense that it only expands nodes for which the cost is smaller than the goal \cite{DP83}. Unfortunately, in the worst case the memory use of \astar{} grows linearly with its running time, which can be a bottleneck for large problems.
Thus, it is often preferable to have algorithms with a smaller memory footprint, possibly at the expense of an increased search time. The best one can expect from a complete algorithm is that the memory depends linearly on the solution depth~\cite{korf1985ida}.

The landmark algorithm for low-memory search is \ida{} \cite{korf1985ida}. At the same time, in problem domains that can be represented by a tree, \ida{} is guaranteed to expand at most $O(\Nopt^2)$ nodes, where $\Nopt$ is the number of expansions required by \astar{} with ties broken for the worst case, assuming the search space is a tree. Although such theoretical guarantees are welcome, in many problems \ida{} really does expand $\Omega(\Nopt^2)$ nodes. This occurs especially in domains with a large diversity of edge costs \cite{sarkar1991reducing}.
Several methods have been developed to mitigate this issue~\cite{burns2013idacr,sharon2014eda,hatem2018solving,sarkar1991reducing,wah1994rida}, but theoretical guarantees (when provided) depend on strong assumptions such as uniformity of the costs or of the branching factor
\cite[for example (cf. Assumptions 1, 2 and 3)]{hatem2015rbfscr}. Another low-memory search algorithm is
RBFS
\cite{1993rbfskorf}, which has the same worst-case number of node expansions as \ida{}.

There is a large literature on algorithms that accept a memory budget as a parameter \cite[and many others]{sen1989fast,Rus92,ZH03}.
Our focus, however, is on lightweight algorithms for which the memory footprint grows linearly with the search depth.
Although we restrict our analysis to tree search, we expect that techniques for graph search, such as transposition tables~\cite{ReinefeldM94}, can be incorporated into our approach.

Like \astar{}, our algorithms use a cost function $\fcost$ on the nodes that we assume to be nondecreasing from parent to child.
Importantly, and in contrast to previous works, we make no regularity assumptions on the search tree, such as polynomial or exponential growth, or on the distribution of $\fcost$-values in the search tree.

We propose two novel parameter-free algorithms that enjoy theoretical guarantees on the number of node expansions relative to \astar{} for tree search and use memory that grows linearly with the search depth.
The first algorithm, \zoomer{}, expands at most $O(\Nopt\log(\thresh^*/\fgapmin))$
nodes, where $\thresh^*$ is the cost of an optimal solution
and $\fgapmin$ is the minimum difference of cost thresholds between two iterations of \ida{}---it is bounded below by the numerical precision, and for integer costs $\fgapmin=1$.
The log factor corresponds to the number of bits necessary to encode
$\thresh^*$ within precision $\fgapmin$.
The central idea of this algorithm is to cap the number of node expansions at $2^k$ nodes at iteration $k$, and perform at each such iteration a binary search for the cost threshold.

The second algorithm, \zzzlong\ (\zzz),
interleaves the iterations of \zoomer{} instead of performing them sequentially.
This is achieved by using a scheduler that may be of independent interest.
If \zoomer{} finds the solution in $M$ node expansions, \zzz{} will not take more than a factor $\log M$ compared to \zoomer,
and removes the explicit dependence on $\fgapmin$ to replace it with the difference $\fgapopt$ of $\fcost$ cost between the solution cost and the next node of higher cost.

Empirical results in heuristic search domains and in two domains we introduce show that although previous algorithms can perform well in some domains, they can fail in others, depending on the underlying structure of the problem. By contrast, both \zzz\ and \zoomer\ are robust to the structure of the problems, performing well in all domains.

\section{Notation and Background}\label{sec:notation}
Let $\Naturals \defeq \{1, 2, 3\ldots\}$,
$\Nonnegints \defeq \{0, 1, 2\ldots\}$
and $\ceilingpos{x} \defeq \max (0, \ceiling{x})$.
The set of all nodes in the underlying search tree describing the problem is $\nodeset$, which may be infinite.
$\children(\node) \subset \nodeset$ is a function that returns the set of children of a node $\node$.
The maximum branching factor of the search tree is $\branch\defeq\max_{\node\in\nodeset} |\children(\node)|$.
We are given a cost function, $\cost:\nodeset\to(0, \infty)$
and assume that $f$ is nondecreasing so that $f(m) \geq f(n)$ for all $n \in \nodeset$ and $m \in \children(n)$.
Let $\nodeset(\thresh) \defeq \{\node\in\nodeset: \cost(\node) \leq \thresh\}$
be the set of nodes for which the cost does not exceed $\thresh$.
Let
$\openset(\thresh) \defeq \cup_{\node\in\nodeset(\thresh)} \children(\node) \setminus \nodeset(\thresh)$ be the nodes at the fringe.
The set of all solution nodes is $\goalset\subseteq\nodeset$.
The cost of an optimal solution is $\thresh^* \defeq \min_{n \in \goalset} \cost(n)$.
Let $\nodecount(\thresh) \defeq |\nodeset(\thresh)|$ and
$\Nopt \defeq \nodecount(\thresh^*)$.
Define $\fgap(\thresh) \defeq \min_{\node\in\openset(\thresh)}\cost(\node) - \thresh$,
which is strictly positive by the assumption that the $\fcost$ cost is nondecreasing.
Furthermore, $\nodecount(\thresh') = \nodecount(\thresh)$ for all $\thresh' \in [\thresh, \thresh + \fgap(\thresh))$.
Let
$\fgapmin{} \defeq \min_{n \in \nodeset, \fcost(\node)\leq\thresh^*} \delta(\cost(n))$, which corresponds to the minimum difference in cost thresholds between two iterations of \ida{} before the optimal solution.
Finally, let $\fgapopt \defeq \fgap(\thresh^*)$.

\begin{remark}
In many definitions of tree search the edges in the tree are associated with costs and $\fcost = \gcost + \hcost$ where $\gcost(n)$ is the cumulative cost from the root to $n$ and $\hcost : \nodeset \to \Reals$ is a consistent heuristic, which guarantees that $\cost$ is nondecreasing.
In case $\hcost$ is admissible but not consistent, $\fcost$
can be made monotone nondecreasing using pathmax~\cite{mero1984pathmax,FelnerZHSSZ11}.
In our analysis it is convenient to deal only with $\fcost$, however.
\end{remark}

\paragraph{Iterative deepening}

\begin{figure}[bt!]
\begin{lstlisting}[numbers=none,label=lst:deepening,caption={Depth-First Search (DFS) starting at the given node with a cost bound $\thresh$
and an expansion budget \texttt{N}.
Returns \texttt{"budget exceeded"} when the budget of node expansions is exceeded,
otherwise returns the descendant solution of minimum cost if one exists, or none otherwise.
It also returns the maximum cost $\threshdown$ below $\thresh$ of expanded nodes,
and the minimum cost $\threshup$ above $\thresh$ of non-expanded nodes among the visited nodes.
}]
def (*\deepening*)(node, $\thresh$, N):
  c := $\fcost$(node)
  if c > $\thresh$: return none, 0, $-\infty$, c
  if is_goal(node): return node, 0, c, $+\infty$
  if N == 0: return "budget exceeded", 0, c, $+\infty$

  # Node expansion
  n_used := 1 # number of expansions (*\label{dfs:nused}*)
  best_desc = none # best descendant solution
  $\threshup$ := $+\infty$ # min cost among nodes in fringe
  $\threshdown$ := c # max cost among nodes expanded
  for child in $\children$(node):
    res, m, $\threshdown_2$, $\threshup_2$ := (*\deepening*)(child, $\thresh$, N - n_used)
    n_used += m
    $\threshdown$ := max($\threshdown$, $\threshdown_2$)
    $\threshup$ := min($\threshup$, $\threshup_2$)
    if res is a Node and ( best_desc is none or $\fcost$(res) < $\fcost$(best_desc) ):
      best_desc := res # better solution found
      $\thresh$ := $\fcost$(res) # branch and bound
    elif res == "budget exceeded":
      return res, n_used, $\threshdown$, $\threshup$

  return best_desc, n_used, $\threshdown$, $\threshup$
\end{lstlisting}
\end{figure}

An iterative deepening tree search algorithm makes repeated depth-first searches (\deepening) with increasing cost threshold. \ida{} and its variants are all based on this idea, as are our algorithms, but with a few new twists.

The pseudo-code for \deepening{} is given in \cref{lst:deepening}, which is
largely the classic implementation
with branch-and-bound optimization to avoid visiting provably suboptimal paths.
It returns an optimal solution or no solution, the number of nodes expanded,
the maximum cost $\threshdown$ among the visited nodes whose costs do not exceed $\thresh$,
and the minimum cost $\threshup$ of the visited nodes whose costs exceed $\thresh$.
Thus we have $\threshdown \leq \thresh < \threshup$.
A node is said to be \emph{expanded} when \deepening{} passes through \cref{dfs:nused}.

Slightly less usual, it also takes as an argument a budget on the number of nodes to expand, and immediately terminates with
\outabudget{} when too many nodes have been expanded.
If the budget is not exceeded, then
$\threshdown = \max\{\cost(\node):\node\in\nodeset(\thresh)\}$
and $\threshup = \min\{\cost(\node):\node\in\openset(\thresh)\}$.
Observe that if during one call to \deepening{}
a suboptimal solution is found and the search budget is not exceeded, then necessarily an optimal solution will be returned.
This optimality property is transferred to all the algorithms presented in this paper.

\ida{} always calls \deepening{} with unlimited budged and it starts with threshold $\theta = \cost(\texttt{root})$, subsequently using $\theta = \theta^+$ from the previous iteration.

If the cost $\thresh^*$ of the optimal solution were known in advance, then a single call to \deepening{} with threshold $\thresh^*$ and an unlimited budget would find the optimal solution with no overhead relative to \astar{} in trees. Of course the optimal cost is usually unknown, which is overcome by calling \deepening{} repeatedly with increasing thresholds.
Algorithms of this type over-expand relative to \astar{} for two reasons. First, in early iterations they re-expand many of the same nodes. Second, in the final iteration when $\thresh \geq \thresh^*$, they tend to overshoot. Overshooting is generally more costly than undershooting because trees usually grow quite fast, which is why \ida{} increases $\thresh$ in the most conservative manner possible.

To illustrate a typical case, suppose that \deepening{} is called with an unlimited budget and threshold $\thresh = \thresh^* + c$ for $c \geq \fgapmin$.
Since $\fcost$ can be constant with depth,
that the number of nodes $\nodecount(\thresh)$ can be arbitrarily large
compared to $\nodecount(\thresh^*)$.
Even when insisting that $\fcost$ has a minimum edge increment $\sigma$,
the number of nodes would still grow exponentially fast as $\smash{\Nopt\branch^{c/\sigma}}$.
Algorithms that update the threshold heuristically without budget constraints are not protected against serious over-expansion and do not effectively control the number of node expansions in the worst case.

\section{\zoomer}

\begin{figure}[bt!]
\begin{lstlisting}[numbers=none,label=lst:zoomer,caption={The \zoomer\ algorithm.}]
def (*\zoomer*)(root):
  lower := $\cost$(root) # assumed > 0
  # $\nodecount_0$: number of expansions at the root
  # up_min: lower bound on upper
  res, $\nodecount_0$, _, up_min := (*\deepening*)(root, lower, $\infty$) (*\label{zoom:first_dfs}*)
  if res is a Node: return res
  for k := 1, 2, ...:
    upper := $\infty$

    while upper $\neq$ up_min:  (*\label{zoom:while}*)
      if upper == $\infty$:
        $\thresh$ := lower $\times$ 2 # sky is the limit   (*\label{zoom:double}*)
      else:
        $\thresh$ := (upper + lower) / 2  (*\label{zoom:half}*)
      $\thresh$ := max($\thresh$, up_min)

      res, _, $\threshdown$, $\threshup$ := (*\deepening*)(root, $\thresh$, $\nodecount_0 2^k$) (*\label{zoom:other_dfs}*)

      if res is a Node:
        return res # solution found
      elif res == "budget exceeded":
        upper := $\threshdown$ # reduce upper bound (*\label{zoom:reduce_upper}*)
      else: # within budget, no solution found
        lower := $\thresh$ # increase lower bound (*\label{zoom:increase_lower}*)
        up_min := $\threshup$  (*\label{zoom:set_up_min}*)
\end{lstlisting}
\end{figure}

The insight behind  \zoomer{} is that the risk of calling \deepening{} with a large threshold can be mitigated by limiting the number of expanded nodes in each iteration. The pseudocode is provided in \cref{lst:zoomer}.

\zoomer{} operates in iterations $k \in \Nonnegints$.
In the zeroth iteration it makes a single call to \deepening{} with an unlimited budget and threshold $\theta = \cost(\text{root})$ (\cref{zoom:first_dfs}). The number of nodes expanded by this search is denoted by $N_0$, which is usually quite small and by definition satisfies $N_0 \leq N^*$.
In subsequent iterations $k \in \Naturals$ the algorithm makes multiple calls to \deepening{} (\cref{zoom:other_dfs}), all with a budget of $N_0 2^k$,
to perform an exponential search on $\thresh$
(\cref{zoom:double,zoom:half})
to identify whether there exists a feasible solution within the budget (\cref{zoom:while}) \cite{bentley1976expsearch}.

Let $\nodecount_k$ be the total number of nodes expanded by \zoomer{} during iteration $k$, which includes multiple calls to \deepening.

\begin{theorem}\label{thm:zoomer}
Assuming $\fcost$(root) > 0, \zoomer{} returns an optimal solution
after expanding no more nodes than
\begin{align*}
    \sum_{k=0}^\infty \nodecount_k \leq \max\left\{1, 4\zoomfact\right\}\nodecount^*\,,
\end{align*}
where $\zoomfact := \ceiling{\log_2(\thresh^*/\thresh_0)} + \ceiling{\log_2(\thresh^*/\fgapmin)}\,.$ 
\end{theorem}
\begin{proof}

Let $B_k := \nodecount_0 2^k$ be the expansion budget at iteration $k$ and 
define the minimum cost $\thresh_k$ at which the budget is exceeded:
$\thresh_k := \min_{\node\in\nodeset} \{\fcost(\node):\nodecount(\fcost(\node)) > B_k\}\,,$ 
and let $K := \min\{k\in\Nonnegints: B_k \geq \nodecount^*\} = \ceiling{\log_2(\nodecount^*/\nodecount_0)}$
be the first iteration with enough budget to find a solution.

Observe that \deepening{} with budget $B_k$ ensures that when it returns 
with \outabudget{} along with the return value $\threshdown$,
if we call again $\deepening(\text{root}, \threshdown, B_k)$ it will also return with \outabudget{},
which means that $\thresh_k \leq \threshdown$ and so we always have $\upper \geq \thresh_k$ (\cref{zoom:reduce_upper}).
Similarly, the algorithm also ensures that we always have $\llower < \thresh_k$.

\paragraph{Iterations $k< K$}
Now, suppose that at some point we have $\upper - \llower \leq \fgapmin$.
Then we have $\upper \leq \llower + \fgapmin < \thresh_k + \fgap(\thresh_k)$
which by \cref{zoom:reduce_upper} implies that $\upper = \thresh_k$.
We also have $\llower \geq \upper - \fgapmin = \thresh_k - \fgapmin$
which by \cref{zoom:set_up_min} implies that \texttt{up\_min} $= \thresh_k=\upper$, which means that no solution exists in $[\llower, \upper]$
since $\llower + \fgap(\llower)=\texttt{up\_min}=\upper$.
Therefore, iteration $k<K$ terminates no later than when $\upper - \llower \leq \fgapmin$ 
 (\cref{zoom:while}).

For $a_k$ times the algorithm goes through \cref{zoom:double}
before calling \deepening,
and thus \llower{} is doubled until $\thresh=2 \times \llower \geq \thresh_k$
when \deepening{} returns with \outabudget,
at which point \upper{} is set to $\threshdown \leq \thresh$;
Together with 
$\llower < \thresh_k$ (by definition of $\thresh_k$),
this implies that $\upper - \llower \leq \thresh - \llower = \llower \leq \thresh_k$.
Thus, starting at worst from $\thresh_0$,
\begin{align*}
    a_k \leq \min\{d\in\Nonnegints:\thresh_0 2^d \geq \thresh_k \} 
    &= \ceiling{\log_2(\thresh_k/\thresh_0)} \\
    &\leq \ceiling{\log_2(\thresh^*/\thresh_0)} \,.
\end{align*}
Thereafter, for $b_k$ times
the algorithm goes through \cref{zoom:half}
before calling \deepening{} until  $\upper - \llower \leq \fgapmin$.
Hence
\begin{align*}
    b_k \leq \min\{d\in\Nonnegints:\thresh_k/2^d \leq \fgapmin\} &= \ceiling{\log_2(\thresh_k/\fgapmin)}  \\
    &\leq 
    \ceiling{\log_2(\thresh^*/\fgapmin)}\,.
\end{align*}

\paragraph{Iteration $K$}
For iteration $K$ things are slightly different.
We still have $\llower < \thresh^*$
(otherwise a solution would have already been found)
and $\upper \geq \thresh_K$.
Assume that at some point $\upper - \llower \leq  2\fgapopt$.
Since $\upper$ is set, the algorithm goes through \cref{zoom:half}
before calling \deepening{}.
Thus $\thresh = (\upper+\llower)/2 \leq \llower + \fgapopt < \thresh^*+\fgapopt$
and similarly $\thresh \geq \upper - \fgapopt \geq \thresh_K - \fgapopt \geq \thresh^*$ where the last inequality is because there is enough budget for 
$\thresh^*$ by definition of $K$ but not for $\thresh_K$.
This implies that $\thresh \in [\thresh^*, \thresh^*+\fgap^*)$
and thus the call to \deepening{} returns with an optimal solution within budget.

Starting from the lower bound $\thresh_0$, the number of calls to \deepening{} 
after going through \cref{zoom:double}
before $\thresh \geq \thresh^*$ is less than $\ceiling{\log_2(\thresh^*/\thresh_0)}$.
At this point, either $\thresh \in [\thresh^*, \thresh_K)$
and an optimal solution is returned,
or $\thresh \geq \thresh_K$.
In the latter case, $\upper := \threshdown \leq \thresh = 2\llower$,
and since $\llower < \thresh^*$ we have $\upper - \llower < \thresh^*$.
Subsequently the number of calls to \deepening{}
after going through \cref{zoom:half} is at most
\begin{align*}
    \min\{d\in\Nonnegints:{} \thresh^*/2^d \leq 2\fgapopt\}
    &\leq\ceilingpos{\log_2( \thresh^* / (2\fgapopt) )}
    \\
    &\leq  \ceilingpos{\log_2( \thresh^* / \fgapmin )} \,.
\end{align*}

Therefore, over all iterations $k \leq K$ and including the first call to \deepening{},
when $\zoomfact \geq 1$
the number of node expansions is bounded by
\begin{equation*}
    \nodecount_0 + \sum_{k=1}^K (a_k + b_k) B_k 
    \leq \sum_{k=0}^K \zoomfact B_k
    \leq 2\zoomfact \nodecount_0 2^K 
    \leq 4\zoomfact \nodecount^*\,.
    \qedhere
\end{equation*}

\end{proof}

\begin{remark}
Because the ordering of the nodes expanded by \deepening{} is invariant under translation of the cost function $\cost$, if $\fcost(\texttt{root})$ is close to zero or negative, the algorithm can simply be run with cost function $f' = f +1 - \cost(\texttt{root})$.
If $\thresh^*$ is known up to constant factors, then the bound can be improved by translating costs so that $\thresh^* / \thresh_0 = O(1)$.
\end{remark}
\begin{remark}
Consider a class of search trees with increasing solution depth $d$ and suppose that $\thresh^* = O(d)$ and $\fgapmin$ is constant, both of which are typical.
When the tree has polynomial growth, then $\Nopt = O(d^p)$ and
\zoomer{} expands a factor of $O(\log \Nopt)$ more nodes than $\Nopt$. In exponential domains with $\Nopt = O(b^d)$, \zoomer{} expands a factor of $O(\log \log \Nopt)$ more nodes than $\Nopt$.
\end{remark}

\section{Uniform Doubling Scheduling}

Iterations $k < K$ in \zoomer{} can take a long time to terminate when $\fgapmin$ is small, where $K$
is the first iteration of \zoomer\ with enough budget to find a solution. This can be mitigated by interleaving the iterations in a careful manner.
We introduce the Uniform Doubling Scheduler (UDS), which provides the correct interleaving and may be of independent interest.
UDS is inspired by Luby's scheduling algorithm for speeding up randomized algorithms \cite{luby1993speedup}.
Suppose you have access to a countably infinite number of deterministic programs indexed by $k \in \Nonnegints$ with unknown running times $(T_k)_{k=0}^\infty$.
UDS operates in blocks $j \in \Naturals$. In the $j$th block it runs program $k_j$ for $\magicshort(j)$ computation steps where $k_j = \log_2 \magicshort(j)$ and $\magicshort(j)$ is the $j$th value of integer sequence A6519, which are all powers of $2$
(see \cref{tab:a6519} and \cref{lst:scheduling}). Note, the state of programs that are not running in the current block are stored in memory.

\begin{theorem}\label{thm:uds}
The total number of computation steps before any program $k$ halts is at most $(4 + k + \log_2(n)) n 2^k$, where $n = \ceiling{T_k/2^k}$.
\end{theorem}

Provided that program $k$ is not expected to run in time much less than $2^k$, then \cref{thm:uds} shows that the overhead incurred by UDS is not especially severe. The proof of \cref{thm:uds}, provided below, depends on two lemmas.

\begin{table}
\centering
\setlength\tabcolsep{3pt}\small
\begin{tabular}{|rrrrrrrrrrrrrrrrr|} \hline
             $j = $   &1 &2 &3 &4 &5 &6 &7 &8 &9 &10 &11 &12 &13 &14 &15 &16\,\,$\cdots$ \\
  $\magicshort(j)=$   &1 &2 &1 &4 &1 &2 &1 &8 &1 &2 &1 &4 &1 &2 &1 &16\,\,$\cdots$ \\
            $k_j= $   &0 &1 &0 &2 &0 &1 &0 &3 &0 &1 &0 &2 &0 &1 &0 &4\,\,$\cdots$ \\ \hline
\end{tabular}
\caption{Sequence A6519: \url{https://oeis.org/A006519}.}\label{tab:a6519}
\end{table}

\begin{lemma}[\citeauthor{orseau2018single}, \citeyear{orseau2018single}, Lemma 10]\label{lem:sched_index}
For all $j\in\Naturals$ and $k\in\Nonnegints$:
$\magicshort(j) = 2^k
\Leftrightarrow
(\exists! n\in\Naturals: j=(2n-1)2^k)$.
\end{lemma}

\begin{lemma}\label{lem:magic_sum}
$\sum_{i=1}^j \magicshort(i) \leq \frac{j}{2}\left(3+\floor{\log_2 j}\right)$ for all $j \in \Naturals$.
\end{lemma}
\begin{proof}
By \cref{lem:sched_index}, the number of blocks program $k$ has been run for after the scheduler has run $j$ blocks is $n = \max\{n\in\Naturals : (2n-1)2^k \leq j\} = \floor{\left(j/2^k + 1\right)/2}$.
Furthermore, the program of largest $k$ that could run up to $j$ is $\max\{k\in\Nonnegints:2^k \leq j\} = \floor{\log_2 j}$.
That is, at block $j$, exactly $1+\floor{\log_2 j}$ programs have computed for at least one step.
Therefore the total number of computation steps over all started programs after block $j$ is
\begin{equation*}
    \sum_{i=1}^j \magicshort(i) = \sum_{k=0}^{\floor{\log_2 j}} \floor{\frac{j/2^k + 1}{2}}2^k
    \leq \frac{j}{2}\left(3+\floor{\log_2 j}\right)\,. \qedhere
\end{equation*}
\end{proof}

\begin{proof}[Proof of \cref{thm:uds}]
UDS always runs program $k$ in blocks of length $2^k$. By \cref{lem:sched_index}, program $k$ will be run for the $n$th block once $j = (2n - 1) 2^k$. Therefore by \cref{lem:magic_sum}, the total computation before program $k$ halts
is at most
\begin{equation*}
\sum_{i=1}^j \magicshort(i) \leq \frac{j}{2}(3 + \floor{\log_2(j)})
\leq (4 + k + \floor{\log_2(n)}) n 2^k\,. \qedhere
\end{equation*}
\end{proof}

\begin{figure}[bt!]
\begin{lstlisting}[numbers=none,label=lst:scheduling,caption={The Uniform Doubling Scheduling algorithm.}]
def (*\magic*)(j):
  return ((j XOR (j - 1)) + 1) / 2

def uniform_doubling_scheduling(prog):
  states := [] # growable vector
  for j := 1, 2, ...:
    k := $\log_2(\magic($j$))$ # exact
    if k >= len(states):
      state[k] := make_state(states, k)

    state[k] := run(prog(states, k), $2^k$)
    if is_goal(state[k]):
      return state[k]
\end{lstlisting}
\end{figure}


\section{\zzzlong{}}

Using UDS to interleave the iterations of \zoomer{} allows us to replace $\fgapmin$ with $\fgapopt$ in the analysis.
The theoretical price for the improvement is at most a logarithmic factor, with the worst case when $\fgapopt = \fgapmin$. See \cref{lst:zzz} (an optimized version is provided in the supplementary material). The following theorem follows from the analysis in the previous two sections.

\begin{theorem}\label{thm:zzz}
Assuming $\thresh_0 \defeq \cost(\texttt{root})>0$,
\cref{lst:zzz} ensures that an optimal solution is found within a number of node expansions bounded by
\begin{align*}
\sum_{k=0}^{\infty}& \nodecount_k \leq
\nodecount_0 +
2(4 +\ceiling{\log_2 (\Nopt/\nodecount_0)}+ \log_2 (\zzzfact))\zzzfact\Nopt\,,
\end{align*}
where $\zzzfact := \ceiling{\log_2(\thresh^*/\thresh_0)} + \ceilingpos{\log_2(\thresh^*/(2\fgapopt))}$.
\end{theorem}
\begin{proof}
\cref{lst:zzz} terminates at the latest when
program $K = \ceiling{\log_2 (\Nopt/\nodecount_0)}$ terminates,
since $2^K\nodecount_0 \geq \Nopt$.
From the proof of \cref{thm:zoomer} for the last iteration $K$ (only),
we know that
the number of calls to \deepening{}
after going through \cref{zzz:double}
before $\thresh \geq \thresh^*$ is less than $\ceiling{\log_2(\thresh^*/\thresh_0)}$.
Subsequently the number of calls to \deepening{}
after going through \cref{zzz:half} is at most
$\ceilingpos{\log_2( \thresh^* / (2\fgapopt) )}$.
Hence the number $n_K$ of blocks of size $2^K$ macro steps (each of $\nodecount_0$ steps) performed by program $K$ before terminating
is bound by $n_K \leq \zzzfact$.
Thus, applying \cref{thm:uds} with $n=n_K$ and $k= K$
gives that the number of macro steps performed by the scheduler is
at most
\begin{multline*}
    (4 + K + \floor{\log_2(n_K)})n_K2^K \\
    \leq 2(4 + \ceiling{\log_2(\Nopt/\nodecount_0)} + \floor{\log_2 \zzzfact})\zzzfact \Nopt/\nodecount_0\,.
\end{multline*}
Multiplying by the number $\nodecount_0$ of steps per macro step
and adding the first call to \deepening{} (\cref{zzz:first_dfs})
leads to the result.
\end{proof}

\begin{figure}[bt!]
\begin{lstlisting}[label=lst:zzz,caption={The \zzz\ algorithm (simple version). $\cost$(root) is assumed strictly positive.},numbers=none]
def (*\zzzlong*)(root):
  lower := [$\cost$(root)] # growable vector
  upper := [] # growable vector
  res, $\nodecount_0$, _, _ := (*\deepening*)(root, lower, $+\infty$) (*\label{zzz:first_dfs}*) #(*\thelstnumber*)
  if res is a Node: return res
  for j := 1, 2, ...:
    k := $\log_2(\magic($j$))$ # exact
    if k >= len(lower):
      lower[k] := $\cost$(root)
    if k >= len(upper):
      $\thresh$ := 2 $\times$ lower[k]  (*\label{zzz:double}*)
    else:
      $\thresh$ := (upper[k] + lower[k]) / 2  (*\label{zzz:half}*)
    res, _, _, _ := (*\deepening*)(root, $\thresh$, $\nodecount_0 2^k$)
    if res is a Node:
      return res # solution found
    elif res == "budget exceeded":
      upper[k] := $\thresh$ # set or reduce upper bound
    else:
      lower[k] := $\thresh$ # increase lower bound
\end{lstlisting}
\end{figure}

\begin{remark}
Sometimes an iteration $k > K$ terminates earlier than iteration $K$. The analysis can be improved slightly to reflect this, as we discuss in the supplementary material.
\end{remark}

\begin{remark}
\cref{lst:zzz} is simple but a little wasteful.
See the supplementary material for an optimized version
that terminates iterations $k<K$ for which the budget is provably insufficient.
\end{remark}

\section{Experiments}\label{sec:experiments}

\begin{figure*}[tb!]
    \centering
    \includegraphics[width=0.255\textwidth]{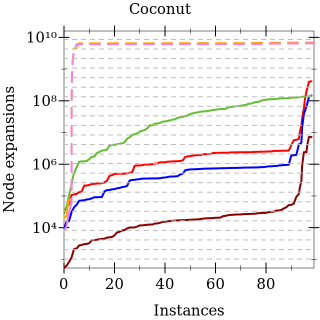}
    \includegraphics[width=0.24\textwidth]{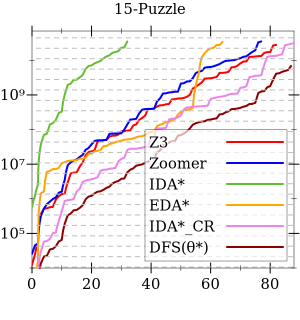}
    \includegraphics[width=0.24\textwidth]{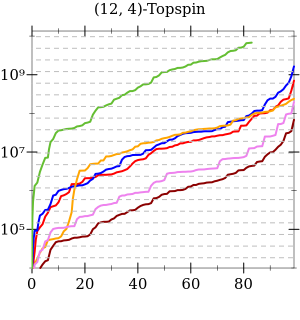}
    \includegraphics[width=0.24\textwidth]{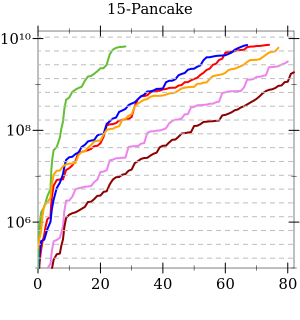}
    \includegraphics[width=0.255\textwidth]{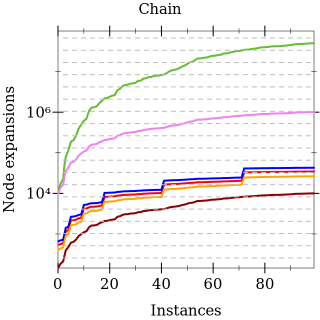}
    \includegraphics[width=0.24\textwidth]{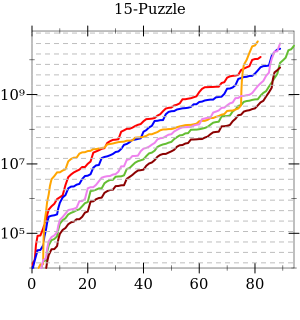}
    \includegraphics[width=0.24\textwidth]{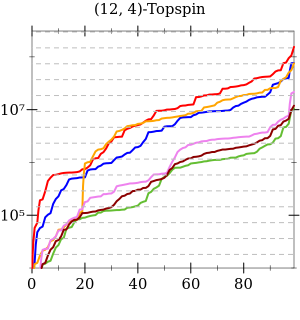}
    \includegraphics[width=0.24\textwidth]{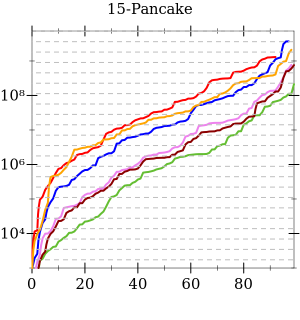}
    \caption{Profiles of node expansions on log scale for 100 instances on four different domains.
    Values are sorted for each solver individually.
    Lines that terminate early mean the time limit is reached for all subsequent instances.
    Dashed gray lines are powers of 2.
    Other dashed lines, when displayed, show the number of node expansions
    when the algorithm is terminated before a solution is found.
    The 3 r.h.s. plots are with random-costs operators on the first row,
    and with unit-costs operators on the second row.
    }
    \label{fig:results}
\end{figure*}
\paragraph{Algorithms tested} We test \ida~\cite{korf1985ida}, \zoomer, \zzz{} (optimized), \eda{}~\cite{sharon2014eda}, and \idacr{}~\cite{sarkar1991reducing}.
\eda{}($\gamma$) is a variant of \ida{} that repeatedly calls \deepening{} with unlimited budget and increasing thresholds. In the $k$th iteration it uses a threshold of $\gamma^k$ where $\gamma > 1$.
In our experiments we take $\gamma = 2$.
In a given iteration, \idacr{} collects the costs of nodes in the fringe into buckets. It then selects the cost for its next iteration based on the information stored in the buckets. The cost is chosen such that
\idacr{} is likely to
expand more than $b^k$ nodes in the $k+1$-th iteration. Similarly to \eda{}, we set $b=2$.
%
We also present the results of \deepening($\theta^*$), which is \deepening\ with its initial bound set to the optimal cost $\theta^*$. Our goal is to verify to how close each algorithm is from \deepening($\theta^*$). For a given heuristic function and without duplicate detection, the number of nodes expanded by \deepening($\theta^*$) is a lower bound on the number of nodes expanded by the algorithms we evaluate.
%

\paragraph{Problem domains}
We use three traditional problem domains in our experiment: 15-Puzzle~\cite{PSS07,DM66}, (12, 4)-TopSpin~\cite{Lam89}, and the 15-Pancake puzzle~\cite{Dwe77}. All these domains were implemented in PSVN~\cite{holte-tr2014}. We use pattern database heuristics (PDBs)~\cite{CulbersonSchaeffer1996} for all three domains. The PDBs for the 15-Puzzle and (12, 4)-TopSpin are generated by projecting tiles 6--15 and tokens 8--12, respectively. The PDB for the 15-Pancake is generated with domain abstraction by mapping pancakes 2--8 to pancake 1 (see the PSVN manual).
We show results both when operator costs are unitary, as usual,
and also when they are chosen uniformly at random in $[1..10000]$ in order
to test the robustness of the algorithms.
We also include the Chain problem, where the search tree is a simple chain (branching factor is 1) and the solution is placed at a depth chosen uniformly at random between 1 and 10000.

The branching factor of the search tree of the domains described above does not vary much. Moreover, all three domains have a relatively large density of solutions; there is a solution in any subtree of the search tree. These properties make the problem of overshooting less pronounced. In order to evaluate the robustness of the search algorithms, we introduce the Coconut problem, which is a domain with varied branching factor and small solution density.
In the Coconut problem we set the heuristic $\hcost$ to 0 and define $\cost = \gcost$ as follows.
The root has $\fcost$ cost 1.
There are $\branch$ coconut trees, and only one has a coconut in its branches.
The agent must climb the trees and find the coconut.
The $\branch$ trunks all have the same size $D$,
sampled uniformly in $[1..10000]$.
Moving along one trunk (using the same action repeatedly) costs 1, jumping from one trunk to another
costs $2D$.
After the trunk come the branches, where each branch is a $\branch$-ary tree.
Moving along the branches or jumping from one branch to another costs $1/10$.
The depth of the coconut in the branches is sampled from a geometric distribution with parameter $1/4$.

\paragraph{Experimental setup}
We use 100 instances for each problem domain and report the number of nodes expanded and number of problems solved with a time limit of 5 hours for the PSVN domains, 1 hour for the Coconut problem and no limit for the Chain problem.
We use the number of nodes expanded instead of running time because all search algorithms expand roughly the same number of nodes per second. The results are shown in \cref{fig:results}, where the $y$-axis shows the number of nodes expanded in log scale with the dashed lines being powers of 2. The $x$-axis depicts the instances from easiest to hardest. To ease visualization, the instances in the $x$-axis are sorted independently for each algorithm. The curves are thus not directly comparable for a particular instance, but they allow one to compare the growth in difficulty for the various solvers.

\paragraph{Discussion}
Only \zoomer\ and \zzz\ perform reliably across all domains.
\ida{} performs very well on domains with unit costs and a branching factor greater than 1, sometimes terminating before $\deepening(\thresh^*)$
because it stops as soon as it finds a solution.
But due to its conservative selection of the threshold, it also expands too many nodes on the other domains.
\idacr\ performs best on the 15-Puzzle, (12, 4)-Topspin and 15-Pancake
with random operator costs.
On these same problems, \eda{}, \zoomer{} and \zzz{} perform similarly.
However, both \idacr{} and \eda{} perform very poorly on the Coconut problem, each of them solving only 4 instances, and are outperformed even by \ida{}.


To explain the behaviour of \idacr\ and \eda\ on the Coconut problem, consider an instance
where the depth of the trunk is $D=2690$ and the depth of the coconut in the branches is $q=6$. The cost set by \eda(2) in the last iteration is 4096, resulting in a search tree with approximately $3^{(4096-2690)/0.1}\approx 10^{6700}$ nodes.
\idacr{} is not saved by choosing a threshold in the range of costs observed in the fringe: Due to the large cost for jumping from one tree to the other,
the threshold in the last iteration can be as large as $2\times 2689$.
This is not a carefully selected example. On the contrary, care must be taken to choose a Coconut problem for which these algorithms work.
Even \eda(1.01) would be only marginally better.
For comparison, \astar{} and \ida{} expand about $10^3$ and $10^6$ nodes respectively on this instance.


By contrast, \zoomer{} and \zzz{} perform both well on the Coconut problem: even though they require a non-negligible factor compared to $\deepening(\thresh^*)$, it appears this factor is independent of the difficulty of the instance, in line with the theoretical results.
\zoomer{} performs slightly better than \zzz\ because the
gaps $\{ \delta(\cost(n)) : n \in \nodeset\}$ are relatively large.

\section{Conclusion}

We derived two linear-memory heuristic search algorithms that require no parameter tuning and come with guarantees on the number of node expansions in the worst case relative to \astar.
\zoomer{} is guaranteed to find optimal solutions in $O(\Nopt\log(\thresh^*/\fgapmin))$ node expansions,
where $\thresh^*$ is the cost of the solution and $\fgapmin$
is the smallest difference in cost between any two nodes that may not be on the same branch.
The second algorithm, \zzzlong{} (\zzz),
expands at most logarithmically more nodes than \zoomer{} in the worst case, but replaces $\fgapmin$ with a `local version' $\fgapopt$ for which $\fgapopt \gg \fgapmin$ often holds.
Theoretical guarantees are summarized in \cref{tab:alg_props}.
\zoomer{} and \zzz{} perform very robustly in all domains tested, whereas all other tested algorithms perform poorly in at least one domain.
\begin{table}[htb!]
    \caption{
    All algorithms return a minimal-cost solution and require memory
    linear with the depth of the search;
    $\zoomfact = O(\log (\thresh^*/\fgapmin))$
    and $\zzzfact = O(\log (\thresh^*/\fgapopt))$.
    }
    \label{tab:alg_props}
    \centering
\begin{tabular}{ll}
\toprule
Algorithm & Worst case  \\
\midrule
\eda      & $\Omega(\branch^{\Nopt})$ \\
\idacr    & $\Omega(\branch^{\Nopt})$ \\
\ida      & $\Omega(\Nopt^2)$   \\
\zoomer   & $O(\zoomfact \Nopt)$  \\
\zzz      & $O(\Nopt\zzzfact\log(\Nopt\zzzfact))$ \\
\bottomrule
\end{tabular}
\end{table}

Although \zoomer{} and \zzz{} are not the fastest linear-memory heuristic search algorithms for all problems, they do perform well consistently and their robustness makes them a safe choice. More aggressive algorithms sometimes perform better on certain problems, but can also fail catastrophically, as evidenced by the Coconut problem.
When prior knowledge is available about the structure of the tree it may be preferable to run a safe algorithm like \zoomer{} or \zzz{} in parallel with a more aggressive choice.

The are many interesting directions for future research, including
the use of transposition tables to prevent node re-expansion on graphs \cite{ReinefeldM94}.

\section*{Acknowledgements}
Many thanks to
Csaba Szepesv\'ari,
Andr\'as Gy\"orgy,
J\'anos Kram\'ar,
Roshan Shariff,
and the anonymous reviewers
for their feedback on this work.

\bibliographystyle{named}
\bibliography{ijcai19}

\ifdefined\suppmat{}

\clearpage

\appendix

\section{Generalized analysis of \zzzlong{}}

The bound in \cref{thm:zzz} can be improved slightly. Let $K = \min\{k : N_0 2^k \geq N^*\}$. All iterations $k < K$ will never find the solution. Iterations $k \geq K$ will eventually find the solution. In the analysis of \zzz{} we simply bounded the number of node expansions of iteration $K$, but it can happen that an even larger $k > K$ will find the solution earlier. Precisely, the $k$th program will halt once it calls \deepening{} with $\theta \in [\theta^*, \theta_k]$ with $\theta_k = \max\{\theta : N(\theta) \leq N_0 2^k\}$.
The number of nodes expanded before iteration $k$ halts is
\begin{align*}
    T_k \leq N_0 2^k \omega(\theta^*, \theta_k)\,,
\end{align*}
where $\omega(a, b)$ is the number of `zooms' before iteration $k$ calls \deepening{} with $\theta \in [a, b]$, which satisfies
\begin{align*}
\omega(\theta^*, \theta_k) =
O\left(\log\left(\frac{\theta^*}{\theta_0}\right) + \log\left(\frac{\theta^*}{\theta_k - \theta^*}\right)\right)\,.
\end{align*}
Then by \cref{thm:uds}, \zzz{} will find the solution after
\begin{align*}
    O\left(\min_{k \geq K} T_k \log(T_k)\right)
\end{align*}
node expansions.
Exactly which program $k \geq K$ minimizes $T_k \log(T_k)$ depends on the blowup of $N(\thresh)$ about $\thresh = \thresh^*$.

\section{Optimized \zzz}

\Cref{lst:zzz_fast} features a few logical optimizations compared to \cref{lst:zzz}.
With a little work, they can be used to improve the bound
of \cref{thm:zzz}, but we will not do this here.

The first improvement is to stop the search
using \texttt{up\_min} as is done for \zoomer,
since once the interval $\texttt{upper}-\texttt{lower}$ is smaller
than the gap $\fgap(\texttt{lower})$, provably no solution can be found.

The second one is to replace the individual lower bounds with a global lower bound
that any program can raise: indeed, when a program raises the lower bound
this means that it has explored all nodes below it without success
and thus no other program can find a solution below that cost either.
Then, to
stop testing the programs that have been proven to not having enough budget to find a solution,
we simply skip the indices $j$ of these programs,
using the properties of \magic,
so that skipped programs take zero computation time.

\paragraph{More improvements}
A further improvement (valid for \zoomer{} as well)
would gather the cost of a solution if one is found but the budget is exceeded.
This cost could then be used as a global upper bound.
This would require modifying the \deepening{} algorithm.

For most iterative deepening based algorithms,
It is also common to avoid re-evaluating whether a node is a solution.
This can be done by passing the previous threshold $\thresh$ to \deepening{}
and evaluating only nodes of a larger cost.
The number of evaluation calls would then only be equal to the number of nodes expanded during the call to \deepening{} with the largest threshold.

In \deepening{}, when a solution has been found, children of cost equal to the solution cost need not be expanded. It is not clear however if the saving in node expansions would be worth the computation cost of the test.

Detection of duplicate states can be performed along the current trajectory to avoid loops in the underlying graph, while keeping a memory that grows only linearly with the depth of the search, but is now a multiple of the size of a state.

\begin{figure}[bt!]
\begin{lstlisting}[label=lst:zzz_fast,caption={The optimized \zzz\ algorithm.},numbers=none]
def (*\zzzlong*)V2(root):
  upper := [] # growable vector
  lower := $\cost$(root)
  res, $\nodecount_0$, _, up_min := (*\deepening*)(root, lower, $+\infty$)
  if res is a Node: return res
  kmin := 0
  j := 0
  repeat forever:
    # Skip steps with provably no solution
    j += $2^\text{kmin}$
    k := $\log_2(\magic($j$))$ # exact
    if k < len(upper):
      if upper[k] <= up_min:
        # Will skip all progs <= k
        kmin := k+1
        # Will move to next factor of $2^{\texttt{kmin}}$
        j -= $2^\texttt{k}$
        continue
      $\thresh$ := (upper[k] + lower) / 2
    else:
      $\thresh$ := 2 $\times$ lower
    $\thresh$ := max($\thresh$, up_min)
    res, _, $\threshdown$, $\threshup$ := (*\deepening*)(root, $\thresh$, $\nodecount_0 2^k$)
    if res is a Node:
      return res # solution found
    elif res == "budget exceeded":
      upper[k] := $\threshdown$ # set or reduce upper bound
    else:
      # Search terminated within budget without
      # a solution. No program can find a
      # solution below this cost.
      lower := $\thresh$
      up_min := $\threshup$
\end{lstlisting}
\end{figure}

\fi 

\end{document}